\RequirePackage{amsmath}
\documentclass{article} 

\usepackage{amsmath}
\usepackage{amsfonts}
\usepackage{amssymb}
\usepackage{amsthm} 
\usepackage{tablefootnote}
\usepackage{authblk}

\usepackage{epsfig}

\usepackage[ruled, vlined, linesnumbered]{algorithm2e}  
\SetKwInOut{Parameter}{Global parameters}

\theoremstyle{plain}
\newtheorem{definition}{Definition}[section]

\theoremstyle{definition}

\newtheorem{proposition}{Proposition}[section]

\theoremstyle{remark}

\def\Rset{\ensuremath{\mathbb R}} 
\def\Xset{\ensuremath{\mathcal X}} 
\def\R{\ensuremath{\mathcal R}} 
\def\S{\ensuremath{\mathcal S}} 
\def\K{\ensuremath{\mathcal K}} 
\def\TT{\ensuremath{\mathfrak S}} 

\def\E{\ensuremath{\mathbb E}} 
\def\P{\ensuremath{\mathbb P}} 

\def\X{\ensuremath{\mathbf X}} 
\def\x{\ensuremath{\mathbf x}} 
\def\y{\ensuremath{\mathbf y}} 
\def\r{\ensuremath{\mathbf r}} 

\DeclareMathOperator*{\argmin}{arg\,min}
\newcommand\ind[1]{\mathbf{1}_{#1}}
\providecommand{\keywords}[1]{\textbf{\textit{Keywords: }} #1}

\title{Rule Induction Partitioning Estimator: \\
	\large{Design of an interpretable prediction algorithm}}

\author{Vincent Margot, Jean-Patrick Baudry, Frederic Guilloux, Olivier Wintenberger}

\affil{Sorbonne Universit\'es,\\
	Campus Pierre et Marie Curie,\\
	Laboratoire de Probabilit\'es, Statistique et Mod\'elisation,\\
	75005 Paris,\\
	vincent.margot@upmc.fr}

\date{}

\begin{document}

	\maketitle
	
	\begin{abstract}
		RIPE is a novel deterministic and easily \emph{understandable} prediction algorithm
		developed for continuous and discrete ordered data. It infers a model, from a sample,
		to predict and to explain a real variable $Y$ given an input variable $X \in \Xset$
		(features). The algorithm extracts a sparse set of hyperrectangles $\r \subset \Xset$,
		which can be thought of as rules of the form \emph{If-Then}. This set is then
		turned into a partition of the features space $\Xset$ of which each cell is explained as
		a list of rules with satisfied their \emph{If} conditions.
		
		The process of RIPE is illustrated on simulated datasets and its efficiency compared
		with that of other usual algorithms.
	\end{abstract}
	
	
	\section{Introduction}
	To find an easy way to describe a complex model with a high accuracy is an important
	objective for machine learning. Many research fields such as medicine, marketing, or
	finance need algorithms able to give a reason for each prediction made. Until now, a
	common solution to achieve this goal has been to use induction rule to describe cells
	of a partition of the features space $\Xset$. A rule is an \emph{If-Then}
	statement which is understood by everyone and easily interpreted by experts (medical
	doctors, asset managers, etc.).
	We focus on rules with a \emph{If} condition defined as a hyperrectangle of $\Xset$.
	Sets of such rules have always been seen as decision trees,
	which means that there is a one-to-one correspondence between a rule and a generated
	partition cell. Therefore, algorithms for mining induction rules have usually been developed
	to solve the \emph{optimal decision tree} problem \cite{hyafil76}. Most of them use a
	greedy splitting technique \cite{CART, Quinlan93, Friedman08, Dembczynski08} whereas
	others use an approach based on Bayesian analysis \cite{Chipman98, Letham15, Yang17}.
	
	RIPE (Rule Induction Partitioning Estimator) has been developed to be a \emph{deterministic}
	(identical output for an identical input) and easily \emph{understandable} (simple to
	explain and to interpret) predictive algorithm. In that purpose, it has also been based
	on rule induction. But, on the contrary to other algorithms, rules selected by RIPE are not
	necessarily disjoint and are independently identified. So, this set of selected rules
	does not form a partition and it cannot be represented as a decision tree. This set is
	then turned into a partition. Cells of this partition are described by a set of activated
	rules which means that their \emph{If} conditions are satisfied. So, a same rule can
	explain different cells of the partition. Thus, RIPE is able to generate a fine partition
	whose cells are easily described, which would usually require deeper decision tree and
	less and less  \emph{understandable} rules.
	Moreover, this way of partitioning permits to have cells which are not a hyperrectangles.
	 
	The simplest estimator is the constant one which predicts the empirical expectation of the
	target variable. From it, RIPE searches rules which are \emph{significantly} different. To
	identify these, RIPE works recursively, searching more and more complex rules, from the most
	generic to the most specific ones. When it is not able to identify new rules, it extracts
	a set of rules by an empirical risk minimization. 
	To ensure a covering of $\Xset$, a \emph{no rule satisfied} statement is added to the set.
	It is defined on the subset of $\Xset$ not covered by the union of the hyperrectangles of
	the extracted rules. At the end, RIPE generates a partition spanned by these selected rules
	and builds an estimator. But the calculation of a partition from a set of hyperrectangles
	is very complex. To solve this issue, RIPE uses what we called the \emph{partitioning trick}
	which is a new algorithmic way to bypass this problem.
		
	\subsection{Framework}
	Let $\left(\X, Y\right) \in \Xset \times \Rset$, where $\Xset =\Xset_1\times\cdots\times\Xset_d$,
	be a couple of random variables with unknown distribution $P$. 
	
	\begin{definition} 
	\begin{enumerate}
			\item Any measurable function $g: \Xset \to \Rset$ is called a \emph{predictor} and
			we denote by $\mathbb G$ the set of all the predictors.
			\item The accuracy of $g$ as a predictor of $Y$ from $\X$ is measured by the quadratic
			\emph{risk}, defined by
			\begin{equation}\label{Risk}
			\mathcal L \left( g\right) = \E_{(\X,Y)\sim\P} \left[( g(\X)-Y)^2\right].
			\end{equation}
		\end{enumerate}
	\end{definition}

	From the properties of the conditional expectation, the optimal predictor is the regression function
	(see \cite{Arlot10,DistributionFree} for more details):
	\begin{equation}\label{reg-function}
	g^* := \E\left[Y | \X\right] = \argmin \limits_{g \in \mathbb G} \mathcal L \left( g\right)
	\text{ a.s.} 
	\end{equation}
	
	\begin{definition}
		Let $D_n = \left( (\X_1, Y_1),\dots,(\X_n, Y_n) \right)$ be a sample of independent and
		identically distributed copies of
		$\left(\X, Y\right)$. 
	\end{definition}

	\begin{definition}
		The empirical risk of a predictor $g$ on $D_n$ is defined by:
			\begin{equation}\label{empirical_risk}
			\mathcal L _n  \left( g\right) = \dfrac1n\sum_{i=1}^{n} \left( g(\X_i) - Y_i \right)^2
			\end{equation}
	\end{definition}
	
	Equation~\eqref{reg-function} provides a link between prediction and estimation of the
	regression function. So, the purpose is to produce an estimator of $g^*$ based on a
	partition of $\Xset$ that  provides a good predictor of $Y$. However, the partition must
	be simple enough to be \emph{understandable}.
	
	\subsection{Rule Induction Partitioning Estimator}
	The RIPE algorithm is based on rules. Rules considered in this paper are defined as follows:
	\begin{definition}\label{rule_def} A rule is an \emph{If-Then} statement such that its
		\emph{If} condition is a hyperrectangle $\r =\prod_{k=1}^{d}I_k$, where each
		$I_k$ is an interval of $\Xset_k$.
		
	\end{definition}
	\begin{definition} For any set $E \subset \Xset$, the empirical conditional expectation
		of $Y$ given $X\in E$ is 
		$$\mu(E, D_n):=\dfrac{\sum \limits_{i=1}^n y_i \ind{\x_i \in E}}{\sum \limits_{i=1}^n
			\ind{\x_i \in E}},$$ 
		where, by convention, $\frac00 = 0$.
	\end{definition}
	The natural estimator of $g^*$ on any $\r \subset \Xset$ is the empirical conditional
	expectation of $Y$ given $X \in \r$. A rule is completely defined by its condition $\r$.
	So, by an abuse of notation we do not distinguish between a rule and its condition. 
	
	A set of rules $\S_n$ is selected based on the sample $D_n$. Then $\S_n$ is turned into
	a partition of $\Xset$ denoted by $\K(\S_n)$ (see Section \ref{partition_trick}). To
	make sure to define a covering of the features space the \emph{no rule satisfied}
	statement is added to the set of hyperrectangles. 
	\begin{definition}\label{norule_def}
		The \emph{no rule satisfied} statement for a set or rules $\S_n$, is an \emph{If-Then}
		statement such that its \emph{If} condition is the subset of $\Xset$ not covered
		by the union of the hyperrectangles of $\S_n$.
	\end{definition}
	One can notice that it is not a rule according to the definition \ref{rule_def} because it
	is not necessarily defined on a hyperrectangle.
	
	Finally, an estimator $\hat g^{\S_n}$ of the regression function $g^*$ is defined:
	\begin{equation}
	\hat g^{\S_n}:\left(\x, D_n\right)\in \Xset \times (\Xset \times \Rset)^n \;\mapsto\;
	\mu(K_n(\x), D_n),
	\end{equation}
	with $K_n(\x)$ the cell of $\K(\S_n)$ which contains $\x$. 
	
	The partition itself is \emph{understandable}. Indeed, the prediction $\hat g^{\S_n}(\x,D_n)$
	of $Y$ is of the form "\emph{If rules \dots are satisfied, then $Y$ is predicted by \dots}"
	The cells of $\K(\S_n)$ are explained by sets of satisfied rules, and the values $\mu(K_n(\x),D_n)$
	are the predicted values for $Y$ .
	
	\section{Fundamental Concepts of RIPE}
	RIPE  is based on two concepts, the \emph{partitioning trick}  and the
	\emph{suitable rule}.
	
	\subsection{Partitioning Trick}\label{partition_trick}
	The construction of a partition from a set of $R$ hyperrectangles is time consuming
	and it is an exponential complexity operation and this construction occurs several times
	in the algorithm. To reduce the time and complexity we have developed the
	\emph{partitioning trick}.
	
	First, we remark that to calculate $\mu(K_n(\x),D_n)$, it is not necessary to
	build the partition, it is sufficient to identify the cell which contains $\x$.
	Figure \ref{fig_make_partition} is an illustration of this process. To do
	that, we first identify rules activated by $\x$, i.e that $\x$ is in their
	hyperrectangles (Fig \ref{fig_make_partition}, to the upper left). And we
	calculate the hyperrectangle defined by their intersection (Fig
	\ref{fig_make_partition}, to the lower left). Then, we calculate the union of 
	hyperrectangles of rules which are not activated (Fig \ref{fig_make_partition},
	to the upper right). To finish, we calculate the cell by difference of the intersection 
	and the union (Fig \ref{fig_make_partition}, to the lower right). The generated subset
	is the cell of the partition $\K(\S_n)$ containing $\x$ .
	
	\begin{figure}
		\centering
		\includegraphics[scale=0.35]{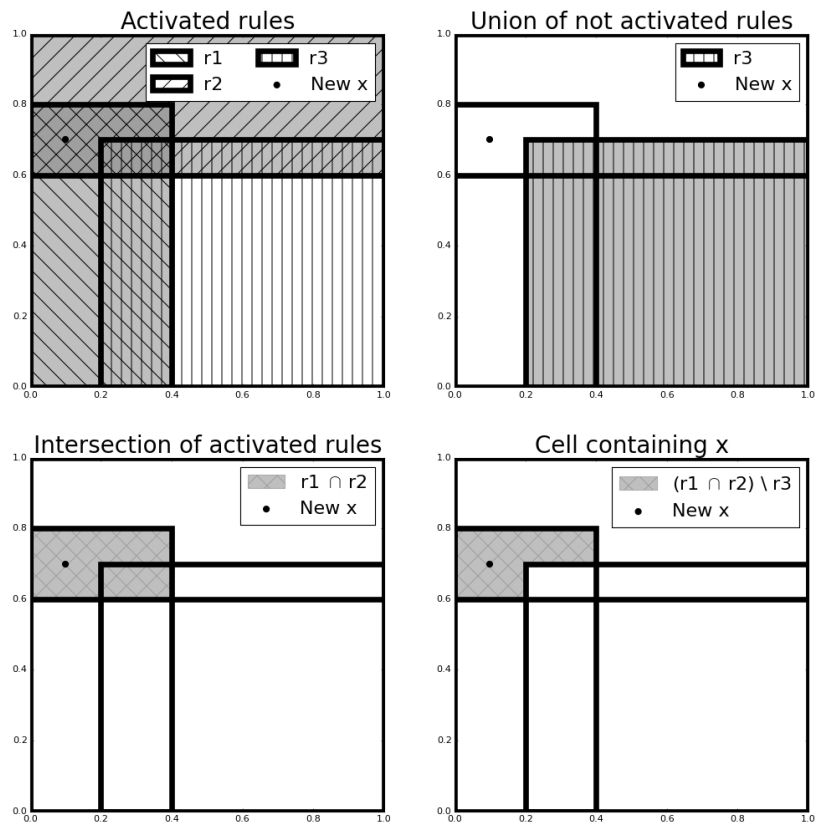}
		\caption{Different steps of the \emph{partitioning trick}
			for a set of three hyperrectangles $\{\r_1, \r_2, \r_3\}$ of $ [0,1]^2$ and a new observation
			$\x = (0.1, 0.7)$. It is important to notice that the cell containing $\x$,
			$(\r_1 \cap \r_2) \setminus \r_3$, is not a hyperrectangle so it does not define a rule.}
		\label{fig_make_partition}
	\end{figure}

	\begin{proposition}
		Let $\S_n$ be a set of $R$ rules selected from a sample $D_n$.
		Then, the complexity to calculate $\mu(K_n(\x),D_n)$ for a new observation
		$\x \in \Xset$ is $O(nR)$.
	\end{proposition}
	
	\begin{proof}
		It is sufficient to notice that $\mu(K_n(\x),D_n)$ can be express as follows:
		\begin{equation}\label{partition_eq}
			\mu(K_n(\x),D_n) = \dfrac{\sum \limits_{j=1}^n y_j k(\x, \x_j, \S_n)} {\sum
				\limits_{j=1}^n k(\x, \x_j, \S_n) }, 
		\end{equation}
		with $k(\x, \x_j, \S_n) =  \prod \limits_{i=1}^R \left( \ind{\x \in \r_i}\ind{\x_j \in \r_i} 
		+ \ind{\x \notin \r_i}\ind{\x_j \notin \r_i} \right)$.
		
		In~\eqref{partition_eq}, the complexity $O(nR)$ appears immediately.
		\hfill $\blacksquare$
	\end{proof}
	
	\subsection{Independent Suitable Rules}
	Each dimension of $\Xset$ is discretized into $m_n$ classes such that
	\begin{equation}\label{modalities}
	\dfrac{(m_n)^d}{n} \to 0 , \qquad n \to \infty.
	\end{equation}
	
	To do so empirical quantiles of each variable are considered (when it has more than $m_n$
	different values). Thus, each class of each variable covers about $100/m_n$ percent of the
	sample. This discretization is the reason why RIPE deals with continuous and ordered
	discrete variables only.
	
	It is a theoretical condition. However, it indicates that $m_n$ must be inversely related
	to $d$: The higher the dimension of the problem, the smaller the number of modalities.
	It is a way to avoid \emph{overfitting}.
	
	\noindent
	We first define two crucial numbers:
	\begin{definition} Let $\r=\prod_{k=1}^{d}I_k$ be a hyperrectangle.
		\begin{enumerate}
			\item The \emph{number of activations} of $\r$ in the sample $D_n$ is 
			\begin{equation}
			n(\r, D_n) := \sum\nolimits_{j=1}^{n} \ind{\x_j \in \r}.
			\end{equation}
			\item The \emph{complexity} of $\r$ is 
			\begin{equation}
			cp(\r)=d-\#\{1\le k\le d; I_k =\Xset_k\},
			\end{equation}
		\end{enumerate}
	\end{definition}
	
	\noindent
	We are now able to define a \emph{suitable rule}. 
	
	\begin{definition}
		A rule $\r$ is a \emph{suitable rule} for a sample $D_n$ if and only if it satisfies the two
		following conditions:
		\begin{enumerate}
			\item \textbf{Coverage condition.}
			\begin{equation}\label{cov_crit} 
			\dfrac{n(\r, D_n)}n \leq \frac 1{\ln(m_n)},
			\end{equation}
			\item \textbf{Significance condition.}
			\begin{equation}\label{sini_crit} 
			\left|  \mu(\r, D_n) -  \mu(\Xset, D_n)\right| \geq z(\r, D_n),
			\end{equation}
			for a chosen function $z$.
		\end{enumerate}
	\end{definition}
	
	The coverage condition~\eqref{cov_crit} ensures that the coverage ratio
	$n(\r, D_n) / n$ of a rule tends
	toward $0$ for $n \to \infty$. It is a necessary condition to prove the consistency of
	the estimator which it is the purpose of a companion paper.

	The threshold in the significance condition~\eqref{sini_crit} is to ensure that the 
	local estimators defined on subsets $\r$ is different than the simplest estimator
	which is the one who is identically equal to $\mu(\Xset, D_n)$.
	\newline

	RIPE generates rules of complexity $c \geq 2$ by a \emph{suitable intersection} of rules
	of complexity $1$ and rule of complexity $c-1$.
	\begin{definition}\label{suitable_intersection}
		Two rules $\r_i$ and $\r_j$ define a \emph{suitable intersection} if and only if
		they satisfy the two following conditions:
		\begin{enumerate}
			\item \textbf{Intersection condition:}
			\begin{equation}\label{inter_crit} 
				\begin{split}
					& \r_i \cap \r_j \neq \emptyset,\\
					& n(\r_i \cap \r_j, D_n) \neq n(\r_i, D_n),\\
					& n(\r_i \cap \r_j, D_n) \neq n(\r_j, D_n)
				\end{split}
			\end{equation}
			\item \textbf{Complexity condition:}
			\begin{equation}\label{cp_crit}
				cp(\r_i \cap \r_j) = cp(\r_i) + cp(\r_j).
			\end{equation}
		\end{enumerate}
	\end{definition}
	The intersection condition~\eqref{inter_crit} avoids adding a useless condition for a
	rule. In other words, to define a \emph{suitable intersection} $\r_i$ and $\r_j$ must not
	be satisfied for the same observations of $D_n$.
	And the complexity condition~\eqref{cp_crit} means that $\r_i$ and $\r_j$ have
	no marginal index $k$; $\ind{k} \subsetneq \Xset_k$, in common of $\Xset$.
	
	\section{RIPE Algorithm}
	We now describe the methodology of RIPE for designing and selecting rules.
	The \emph{Python} code is available at \textbf{https://github.com/VMargot/RIPE}.

	The main algorithm is described as Algorithm~\ref{algo_main}. The methodology is divided
	into two parts. The first part aims at finding all suitable rule and the second one aims
	at selecting a small subset of suitable rules that estimate accurately the objective~$g^*$. 
	\newline
	
	\noindent
	The parameters of the algorithm are:
	\begin{itemize}
		\item $m_n$, the sharpness of the discretization, which must fulfill~\eqref{modalities};
		\item $\alpha\in[0,1]$, which specifies the false rejecting rate of the test;
		\item $z$, the significance function of the test;
		\item and $M\in\mathbb N$, the number of rules of complexity $1$ and $c-1$ used to
		define the rules of complexity $c$.
	\end{itemize}
	
	\begin{algorithm}
		\caption{Main}
		\label{algo_main}
		\SetAlgoLined
		\Parameter
		{
			$m_n$, $\alpha$, $z$ and $M$;
		}
		\KwIn
		{
			\begin{itemize}
				\item $(\X,\y)\in\Rset^{n(d+1)}$: data\vspace{-0.2cm}
			\end{itemize}
		}
		\KwOut
		{
			\begin{itemize}\vspace{-0.2cm}
				\item $\S$: the set of selected rules
			\end{itemize}
		}
		Set ${h}_n = 1 /\ln(m_n)$ the maximal coverage ratio of a rule\;
		$\tilde\X \leftarrow Discretize(\X, {m}_n)$ discretization in $m_n$ modalities\;
		$\R \leftarrow Calc\_cp1((\tilde\X, \y), h_n)$\;
		\For{$c =2, \dots, d$}
		{
			$\R' \leftarrow Calc\_cpc((\tilde\X, \y), \R, c, h_n)$\;
						\uIf{$len(\R') = 0$}
						{
							Break\;
						}
						\uElse
						{
			$\R \leftarrow append(\R, \R') $\;
						}
		}
		$\R\leftarrow Sort\_by\_risk(\R, (\tilde\X,\y))$\;
		$\S \leftarrow Select(\R, (\tilde\X, \y))$\;
		Return $\S$\;
	\end{algorithm}
	
	\subsection{Designing Suitable Rules}
	
	The design of suitable rules is made recursively on their complexity. It stops
	at a complexity $c$ if no rule is suitable or if the maximal complexity $c=d$ is
	achieved. 

	\subsubsection{Complexity $1$:}
	The first step is to find suitable rules of complexity $1$. This part is described as
	Algorithm~\ref{algo_cp1}. First notice that the complexity of evaluating all rules of
	complexity $1$ is $O(ndm_n^2)$.
	
	Rules of complexity $1$ are the base of RIPE search heuristic. So all rules are considered
	and just suitable are kept, i.e rules that satisfied the coverage condition~\eqref{cov_crit}
	and the significance condition~\eqref{sini_crit}. Since rules are considered regardless
	of each others, the search can be parallelized. 
	
	At the end of this step, the set of suitable rules is sorted by their empirical risk
	\eqref{empirical_risk}, $\mathcal L_n(\hat g^{\{\r\}})$, with $\hat g^{\{\r\}}$ the
	predictor based on exactly one rule $\r$.
	
		\begin{algorithm}
		\caption{Calc\_cp$1$}
		\label{algo_cp1}
		\SetAlgoLined
		\Parameter
		{
			$m_n$, $\alpha$ and $z$;
		}
		\KwIn
		{
			\begin{itemize}
				\item $(\X,\y)\in\Rset^{n(d+1)}$: data
				\item $h_n$: parameter \vspace{-0.2cm}
			\end{itemize}
		}
		\KwOut
		{\begin{itemize}\vspace{-0.2cm}
				\item $\R$: the set of all suitable rules of complexity $1$\;
		\end{itemize}}
		$\R \leftarrow \emptyset$\;
		\For{$i =1, \dots, d$}
		{
			$\x_i \leftarrow \X[i]$, the i$^{th}$ feature \;
			\For{$b_{min} = 0, \dots, {m}_n$ }
			{
				\For{$b_{max} = b_{min}, \dots, {m}_n$}
				{
					Set $\r = \prod_{k=1}^{d} I_k$ with 
					$\begin{cases}
					I_k = [0, m_n],\, k\neq i\\
					I_i = \left[b_{min}, b_{max}\right]\\
					\end{cases}$\;
					\uIf{is\_suitable($\r$, $(\X,\y)$, $h_n$, $z$, $\alpha$)}
					{
						$\R \leftarrow append(\R, \r) $\;
					}
				}
			}
		}
		Return $\R$
	\end{algorithm}

	\subsubsection{Complexity $c$:}
	Among the  suitable rules of complexity $1$ and $c-1$ sorted by their empirical risk
	\eqref{empirical_risk}, RIPE selects the $M$ first rules of each complexity ($1$ and $c-1$).
	Then it generates rules of complexity $c$ by pairwise \emph{suitable intersection}
	according to the definition~\ref{suitable_intersection}.
	It is easy to see that the complexity of evaluating all rules of complexity $c$ obtained
	from their intersections is then $O(nM^2)$.
	
	The parameter $M$ is to control the computing time and it is fixed by the statistician.
	This part is described as Algorithm~\ref{algo_upcp}.

	\begin{algorithm}[t]
		\caption{Calc\_cp$c$}
		\label{algo_upcp}
		\SetAlgoLined
		\Parameter
		{
			$\alpha$, $z$ and $M$;
		}
		\KwIn
		{
			\begin{itemize}
				\item $(\X,\y)\in\Rset^{n(d+1)}$: data
				\item $\R$: set of rules  of complexity up to $c-1$
				\item $c$: complexity
				\item $h_n$: parameter\vspace{-0.2cm}
			\end{itemize}
		}
		\KwOut
		{
			\begin{itemize}\vspace{-0.2cm}
				\item $\R_{c}$: set of suitable rules of complexity ${c}$
			\end{itemize}
		}
		$\R_{c} \leftarrow \emptyset$\;
		$\R\leftarrow Sort\_by\_risk(\R, (\X,\y))$\;
		$\R_{1} \leftarrow$ the $M$ first rules of
		complexity $1$ in $\R$\; 
		$\R_{c-1} \leftarrow$ the $M$ first  rules
		of complexity ${c}-1$  in $\R$\;
		\uIf{$\R_{1} \neq \emptyset$ and $\R_{c-1} \neq \emptyset$}
		{
			\For{$\r_1$ in $\R_1$}
			{
				\For{$\r_2$ in $\R_{c-1}$}
				{
					\uIf{is\_suitable\_intersection($\r_1$, $\r_2$)}
					{
						Set $\r = \r_1 \cap \r_2$\;
						\uIf{is\_suitable($\r$, $(\X,\y)$, $h_n$, $z$, $\alpha$)}
						{
							$\R_{c} \leftarrow append(\R_{c}, \r) $\;
						}
					}
				}
			}
		}
		Return $\R_{c} $
	\end{algorithm}
		
	\subsection{Selection of Suitable Rules}\label{selection_method}
	After designing suitable rules, RIPE selects an optimal set of rules.
	Let $\R_n$ be the set of all suitable rules generated by RIPE. The optimal subset 
	$\S_n^* \subset \R_n$ is defined by
	\begin{equation}\label{selection_formula}
		\S_n^*:=\argmin  \limits_{\S \subset \R_n} \mathcal L_n (\hat g^{\{\S\}} )
	\end{equation}
	is the empirical risk~\eqref{empirical_risk} of the predictor based on $\S$.
	
	Each computation of the empirical risk \eqref{selection_formula}
	requires the partition from the set $\S$ of rules, as described in Section
	\ref{partition_trick}.
	The complexity to solve~\eqref{selection_formula} naively, comparing all the possible sets
	of rules, is exponential in the number of suitable rules.
	
	To work around this problem, RIPE uses Algorithm~\ref{algo_select}, a greedy recursive
	version of the naive algorithm: it does not explore all the subsets of $\R_n$. Instead,
	it starts with a single rule, the one with minimal risk, and iteratively keeps/leaves
	the rules by comparing the risk of a few combinations of these rules. More
	precisely, suppose that 
	\begin{itemize}
		\item $\r_1,\dots,\r_N$ are the suitable rules, sorted by increasing empirical risk;
		\item $\r_1,\dots,\r_k,\; k<N,$ have already been tested;
		\item $j$ of them, say $\S\subset\{\r_1,\dots,\r_k\}$ have been kept, the $k-j$ other
		being left.
	\end{itemize}
	Then $\r_{k+1}$ is tested in the following way : 
	\begin{itemize}
		\item Compute the risk of $\S$, $\S\cup\{\r_{k+1}\}$ and of all $\S\cup\{\r_{k+1}\}
		\!\setminus \!\{\r\}$ for $\r\in\S$;
		\item Keep the rules corresponding to the minimal risk;
		\item Possibly leave \emph{once for all}, the rule in $\{\r_1,\dots,\r_{k+1}\}$ which
		is not kept at this stage.
	\end{itemize}
	Thus, instead of testing the $2^N$ subsets of rules, we make $N$ steps and at the $k^{\text{th}}$
	step we test at most $k+2$ (and usually much less) subsets, which leads to a theoretical
	overall maximum of $O(N^2)$ tested subsets. The heuristic of this strategy is that rules
	with low risk are more likely to be part of low risk subsets of rules; and the minimal risk
	is searched in subsets of increasing size.
	
	\begin{algorithm}
		\caption{Select}
		\label{algo_select}
		\SetAlgoLined
		\KwIn
		{
			\begin{itemize}
				\item $\R$: set of rules sorted by increasing risk			
				\item $(\X,\y)\in\Rset^{n(d+1)}$: data\vspace{-0.2cm}
			\end{itemize}
		}
		\KwOut
		{
			\begin{itemize}\vspace{-0.2cm}
				\item $\S$: subset of selected rules approaching the argmin~\eqref{selection_formula}
				over all subsets of $\R$ \;
			\end{itemize}
		}
		Set $\S=\{\R(1)\}$\;
		\For{$i =2, \dots, len(\R)$}
		{
			Set $\TT=\{\;\S\;;\;\S\!\cup\!\{\R(i)\}\;\}$\;
			\For{j=1,\dots,len(\S)}
			{
				$\TT\leftarrow append(\;\TT\;;\;\S\!\cup\!\{\R(i)\}\!\!\setminus\!\!\{\S(j)\}\;)$
			}
			$\TT\leftarrow Sort\_by\_risk(\TT,(\X,\y))$\;
			$\S\leftarrow\TT(1)$
		}
		Return $\S$\;
	\end{algorithm}

	\section{Experiments}
	The experiments have been done with \emph{Python}. To assure reproducibility the \emph{random seed}
	has been set at $42$. The codes of these experiments are available in \textbf{GitHub} with
	the package RIPE.
	
	\subsection{Artificial Data}
	The purpose here it is to understand the process of RIPE, and how it can explain a phenomenon.
	We generate a dataset of $n = 5000$ observations with $d = 10$ features. The target variable
	$Y$ depends on two features $X_1$ and $X_2$ whose are identically distributed on $[-1,1]$.
	In order to simulate features assimilated to white noise, the others variables follow a centered-
	reduced normal distribution $\mathcal{N}(0,1)$. The model is the following
	\begin{equation}\label{model}
	y_i =F^*(\x_i)+\epsilon_i
	\end{equation}
	with $\epsilon_i \sim \mathcal{N}(0,1)$ and
	\begin{equation}\label{nolin_data_y}
		F^*(\x_i) = -2 \times \ind{\left\{x_{1,i}^2 + x_{2,i}^2 > 0.8 \right\}} + 2 \times 
		\ind{\left\{x_{1,i}^2 + x_{2,i}^2 < 0.5 \right\}}
	\end{equation}
	
	The dataset is randomly split into training set $D^1_m$ and test set $D^2_l$ such as $D^1_m$
	represents $60\%$ of the dataset. RIPE uses significance test based on
	\eqref{bernstein_test} with a threshold $\alpha = 0.05$ and $m_n=5$. 
	
	On Figure \ref{nolin_rez}, we have on the left, the true model \eqref{nolin_data_y} according to
	$X_1$ and $X_2$ with the realization of $Y$ not used during the learning. On the right, the model
	inferred by RIPE.
	
	On Tab \ref{nolin_rules} we represent the set of selected rules. In this case rules form
	a covering of the features space. So it is not necessary to add the \emph{no rule satisfied}
	statement (see Def \ref{norule_def}).
	
	\begin{figure}[h]
		\centering
		\includegraphics[scale=0.35]{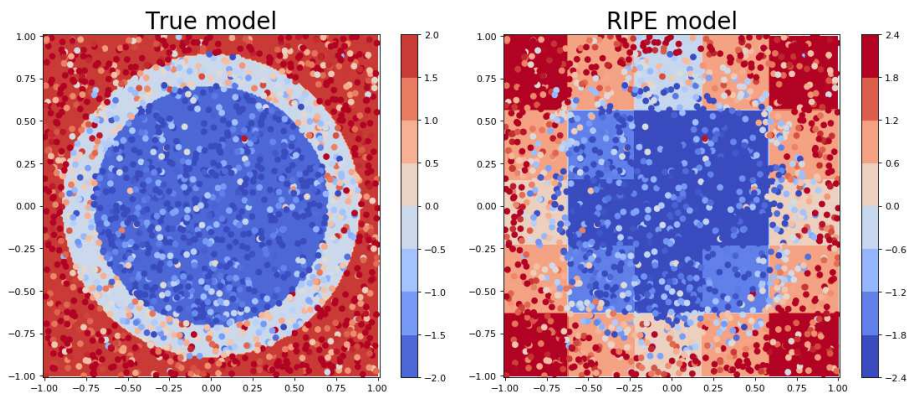}
		\caption{The true model vs the model inferred by RIPE.}
		\label{nolin_rez}
	\end{figure}
	
	\begin{table}
		\centering
		\scalebox{0.9}{
		\begin{tabular}{l l c c c c}
			Rule &  Conditions &  Coverage &  Prediction &     Z &       $MSE$ \\
			\hline
			R 0(2)- & $ X0 \in [1.0, 3.0] \ \& \ X1 \in [1.0, 3.0] $ & $     0.36$ & $      -0.91$ & $ 0.09$ & $ 2.14$\\
			R 1(2)- & $ X0 \in [1.0, 3.0] \ \& \ X1 \in [2.0, 4.0]$ & $     0.36$ & $      -0.57$ & $ 0.09$ & $ 3.31$\\
			R 2(2)- & $ X0 \in [2.0, 4.0] \ \& \ X1 \in [1.0, 3.0]$ & $     0.35$ & $      -0.54$ & $ 0.09$ & $ 3.40$\\
			R 3(1)- & $                       X0 \in [2.0, 3.0]$ & $     0.40$ & $      -0.46$ & $ 0.08$ & $ 3.48$\\
			R 4(1)- & $                       X0 \in [1.0, 2.0]$ & $     0.40$ & $      -0.43$ & $ 0.08$ & $ 3.55$\\
			R 5(1)- & $                       X1 \in [1.0, 2.0]$ & $     0.40$ & $      -0.43$ & $ 0.08$ & $ 3.55$\\
			R 6(1)+ & $                                  X0 = 4.0$ & $     0.20$ & $       0.63$ & $ 0.11$ & $ 3.65$\\
			R 7(1)+ & $                                  X1 = 4.0$ & $     0.20$ & $       0.56$ & $ 0.12$ & $ 3.74$\\
			R 8(1)+ & $                       X1 \in [0.0, 1.0]$ & $     0.40$ & $       0.19$ & $ 0.08$ & $ 3.95$\\
			R 9(1)+ & $                       X0 \in [0.0, 1.0]$ & $     0.40$ & $       0.15$ & $ 0.08$ & $ 3.99$\\
			\hline
		\end{tabular}
		}
		\caption{Summary of selected rules with conditions interval express as modalities}
		\label{nolin_rules}
	\end{table}

	\subsection{High Dimension Simulation}\label{high_dimension_problem}
	In this simulation, we use the function \emph{make\_regression}, from the \textbf{Python}
	package \emph{sklearn} (\cite{scikit-learn}), to generate a random linear regression model with $n$
	observations and $d$ variables. Among these variables, $p$ are informative and the rest are
	gaussian centered noise.
	
	In this example we take $n=500$, $d=1000$ and $p=5$ to simulate a noisy high dimensional problem.
	The data are randomly split into a training set and a test set, with a ratio of $60\% $
	\textbackslash $40\%$, respectively.
	
	We use two others algorithms in this case: Decision Tree (DT) \cite{CART} without pruning
	and Random Forests (RF) \cite{Breiman01}, all from the package
	of python \textbf{sklearn} \cite{scikit-learn}. In order to evaluate the performance of
	our model, the normalized mean square error ($NMSE$) is computed. 
	
	Results are summarized in Tab~\ref{resum_tab}. Difference between the $NMSE$ of the training
	and the $NMSE$ of test indicates that Decision Tree and Random Forests overfit in this context.
	Conversely, RIPE infers a model which is more general (see Tab~\ref{resum_tab}). Indeed,
	RIPE is able to describe the model with only $14$ rules (see Tab~\ref{resum_rules}) which
	have conditions on only seven variables from $1000$. Among these selected variables
	only two are very important $X_{976}$ and $X_{298}$ (see Tab~\ref{var_count}).
	
	In this case, RIPE discretizes each variable in $5$ modalities from $0$ to $4$. Table
	\ref{resum_rules} presents the selected rules with their conditions. The rule $R 14$
	is the \emph{no rule satisfied} statement (see Definition \ref{norule_def}).

	\begin{table}
		\centering
		\scalebox{0.9}{
		\begin{tabular}{l l c c c c}
			Rule &  Conditions &  Coverage &  Prediction &     Z &       $MSE$ \\
			\hline
			R 1(2)-  &  $X_{976} \in [0.0, 2.0]  \ \& \  X_{298} \in [0.0, 2.0] $ & $0.35$ & $-0.83$ &  $0.28$ & $7808.24$ \\
			R 2(1)-  &  $X_{976} = 0.0$ & $0.20$ & $-1.07$ &  $0.46$ &   $8907.38$ \\
			R 3(1)+  &  $X_{976} = 4.0$ & $0.20$ & $0.88$ &  $0.41$ &  $10081.34$ \\
			R 4(2)-  &  $X_{976} \in [0.0, 1.0]  \ \& \  X_{336} \in [0.0, 1.0]$ & $0.19$ & $ -0.89$ &  $0.40$ &  $10245.30$ \\
			R 5(2)+  &  $X_{298} \in [2.0, 4.0]  \ \& \  X_{976} \in [2.0, 3.0] $&      $0.24$ &        $0.65$ &  $0.27$ &  $10781.00$ \\
			R 6(1)+  &  $X_{298} = 4.0 $&      $0.20$ &        $0.73$ &  $0.43$ &  $10813.83$ \\
			R 7(1)-  &  $X_{298} = 0.0 $&      $0.20$ &       $-0.73$ &  $0.42$ &  $10822.09$ \\
			R 8(2)-  &  $X_{298} \in [0.0, 1.0]  \ \& \  X_{336}  \in [0.0, 1.0] $&      $0.20$ &      $ -0.66$ &  $0.38$ &  $11109.50$ \\
			R 9(2)+  & $X_{976} \in [2.0, 4.0]  \ \& \  X_{564}  = 4.0 $&      $0.14$ &        $0.77$ &  $0.45$ &  $11253.10$ \\
			R 10(2)+  & $X_{976} \in [2.0, 4.0]  \ \& \  X_{163}  = 4.0 $&      $0.13$ &        $0.75$ &  $0.44$ &  $11419.93$ \\
			R 11(2)- & $X_{976} \in [0.0, 1.0]  \ \& \  X_{945}  = 2.0 $&      $0.10$ &       $-0.87$ &  $0.51$ &  $11427.16$ \\
			R 12(2)- & $X_{976} \in [0.0, 1.0]  \ \& \  X_{733}  = 1.0 $&      $0.10$ &      $-0.84$ &  $0.58$ &  $11524.60$ \\
			R 13(1)+ &  $X_{976} = 3.0 $&      $0.20$ &        $0.55$ & $ 0.31$ &  $11548.05$ \\
			R 14 & No rule activated  &      $0.02$ &        $-0.35$ & $ 0.45$ &  $12440.40$ \\
			\hline
		\end{tabular}
		}
		\caption{Summary of selected rules with conditions interval express as modalities}
		\label{resum_rules}
	\end{table}
	
	\begin{table}
		\centering
		\scalebox{0.8}{
		\begin{tabular}{| l | p{20mm}|| c | c | c |c |}
			\hline
			Algorithm & Parameters & $NMSE$ training & $NMSE$ test & Nb of rules & Complexity max\\ \hline
			DT & / & $0.0$ & $0.46$ & $350$ & $14$\\ \hline
			RF & m\_tree = $200$  \newline m\_try = $d/3$ & $0.04$ & $0.39$ & $128.25$\tablefootnote{
				It is the mean of the number of rules of each tree} & $21$ \\ \hline
			RIPE &  $M$=300 \newline z: see \eqref{bernstein_test}\newline
			$\alpha$=$0.05$ & $0.13$ & $0.30$ & $14$ & $2$ \\ \hline
		\end{tabular}
		}
		\caption{Performance results of RIPE compared to two supervised learning algorithms:
			The Decision Tree (DT) and the Random Forests (RF).}
		\label{resum_tab}
	\end{table}
		
	\begin{table}
		\centering
		\begin{tabular}{|l||ccccccc|}
			\hline
			Variable &  $X_{976}$ & $X_{298}$ & $X_{336}$ & $X_{163}$ & $X_{945}$ & $X_{565}$ & $X_{733}$ \\ \hline
			Count & $10$  & $5$  & $2$ & $1$  & $1$  & $1$  & $1$  \\ \hline
		\end{tabular}
		\caption{Count of variable occurencies in rules selected by RIPE.}
		\label{var_count}
	\end{table}

	\newpage
	\subsection{Real Data}
	In this section, we present a quick overview of the use of the alogirthm RIPE on the well-known
	Kaggle's\footnote{https://www.kaggle.com/} dataset: \emph{Titanic}.
	It is a binary classification problem. The goal is to predict which passengers survived the tragedy. We have
	kept only $7$ features. We have dropped features \emph{Name}, \emph{Ticket Number}, and \emph{Cabin Number}
	which are considered irrelevant for a first study, and we haven't done data engineering.

	The accuracy rate given by Kaggle for RIPE's predictions on the test set is $0.765$, but the most interesting
	output is the description of the model.

	This can be sum up in the table \ref{titanic_tab} and with the two following figures. Figure \ref{titanic1} shows
	that the most important feature is the \emph{fare} which appears seven times in the set of selected rules. The
	figure \ref{titanic2} permits to be more specific. Indeed, we can notice that the cheaper the ticket, the higher
	the risk to die.

	This example shows the kind of interpretation that RIPE could offer to a statistical study on an unknown dataset.
	
	\begin{figure}[h]
		\centering
		\includegraphics[scale=0.7]{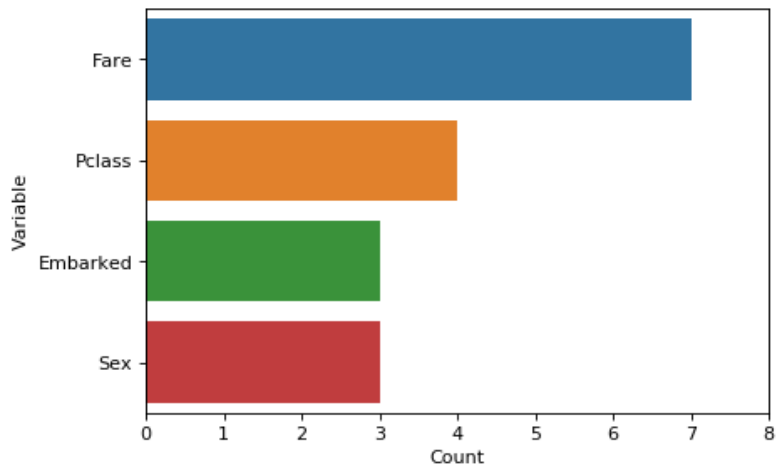}
		\caption{Distribution of rules by variables}
		\label{titanic1}
	\end{figure}

	\begin{figure}[h]
		\centering
		\includegraphics[scale=0.7]{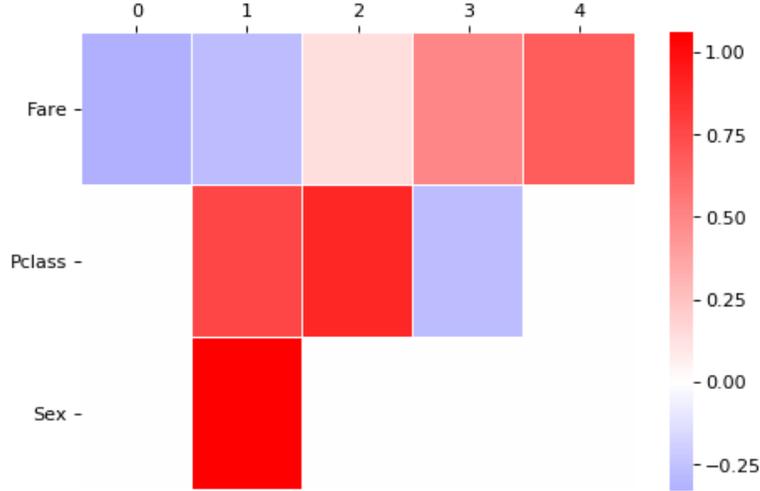}
		\caption{Sum of the prediction of rules for each quantitative variable by modalities}
		\label{titanic2}
	\end{figure}

	\begin{table}[h]
		\centering
		\scriptsize
		\begin{tabular}{l l c c c c}
			Rule &  Conditions &  Coverage &  Prediction &     Z &       $MSE$ \\
			\hline
			R 0(2)+ & $           Sex = female \ \& \  Pclass \in [1.0, 2.0]$ & $ 0.19$ & $        1.16$ & $  0.26$ & $       0.32 $\\
			R 1(2)+ & $                        Sex = female \ \& \  Fare = 4.0$ & $ 0.11$ & $        1.09$ & $  0.35$ & $       0.41 $\\
			R 3(2)+ & $             Sex = female \ \& \  Embarked \in [C, Q]$ & $ 0.12$ & $        0.93$ & $  0.32$ & $       0.42 $\\
			R 2(1)+ & $                                     Pclass = 1.0$ & $      0.24$ & $        0.51$ & $  0.21$ & $       0.43 $\\
			R 5(1)- & $                            Fare \in [0.0, 1.0]$ & $      0.38$ & $ -0.38$ & $  0.14$ & $       0.43 $\\
			R 6(2)+ & $             Pclass \in [1.0, 2.0] \ \& \  Fare = 4.0$ & $ 0.18$ & $        0.63$ & $  0.25$ & $       0.43 $\\
			R 4(2)- & $             Pclass = 3.0 \ \& \  Fare \in [0.0, 2.0]$ & $ 0.47$ & $       -0.27$ & $  0.13$ & $       0.44 $\\
			R 7(2)+ & $  Fare \in [2.0, 4.0] \ \& \  Embarked \in [C, NaN]$ & $      0.15$ & $        0.54$ & $  0.27$ & $       0.45 $\\
			R 8(2)+ & $    Fare \in [2.0, 4.0] \ \& \  Embarked \in [C, Q]$ & $      0.17$ & $        0.45$ & $  0.25$ & $       0.45 $\\
			R 9(1)- & $                            Fare \in [1.0, 2.0]$ & $      0.41$ & $ -0.16$ & $  0.14$ & $       0.46 $\\
			R 10    &  No rule activated & $      0.09$ & $       -0.40$ & $  0.28$ & $       0.47 $\\
			\hline
		\end{tabular}
		\label{titanic_tab}
		\caption{Summary of selected rules with conditions interval express as modalities}
	\end{table}

	\section{Conclusion and Future Work}
	In this paper we present a novel \emph{understandable} predictive algorithm, named RIPE.
	Considering the regression function is the best predictor RIPE has been developed to be
	a simple and accurate estimator of the regression function. The algorithm identified a
	set of \emph{suitable rules}, not necessary disjoint, of the form \emph{If-Then} such
	as their \emph{If} conditions are hyperrectangles of the features space $\Xset$. Then,
	the estimator is built on the partition generated by the \emph{partitioning trick}.
	Its computational complexity is linear in the data dimension $O(dn)$.
	
	RIPE is different from existing methods which are based on a space-partitioning tree.
	It is able to generate a fine partition from a set of \emph{suitable rules}, reasonably
	quickly such that their cells are explained as a list of \emph{suitable rules}. Whereas
	there is a one-to-one correspondence between a rule and a cell of a partition provided
	by a decision tree. So to have a finer partition decision trees must be deeper and rules
	become less and less \emph{understandable}.
	Furthermore, on the contrary to decision trees, the partition generated by RIPE can have cells
	which are not hyperrectangles.
	
	A paper on the universal consistency of RIPE under some technical conditions
	is in preparation.

	\newpage
	
	\section*{Appendix: Examples of Significance Function}\label{appendix}
	
	Here, we present three functions $z$ used in practice.
	\newline
	
	\begin{enumerate}
		\item  The first one is based on the H\oe ffding's inequality \cite{Hoeffding63}:
		\begin{equation}\label{hoeffding_test}
		z(\r, D_n, \alpha) = \dfrac{(M-m) \sqrt{\ln(2 / \alpha)}}{\sqrt{2n(\r, D_n)}},
		\end{equation}
		where $M = \max \limits_{i\in \{1, \dots, n\}} y_i$ and $m = \min \limits_{i\in \{1, \dots, n\}} y_i$.
	
		\item The second one is based on the Bernstein's inequality:
		\begin{equation}\label{bernstein_test}
		z(\r, D_n, \alpha) = \dfrac{1}{6n(\r, D_n)}
		\left( M\ln\left( \dfrac{2}{\alpha}\right) + \sqrt{M^2\ln\left(\dfrac{2}{\alpha}\right)^2
			+ 72 v\ln\left( \dfrac{2}{\alpha}\right)}\right),
		\end{equation}
		where $M = \max \limits_{i\in \{1, \dots, n\}} y_i$ and $v = \sum_{i=1}^{n} y_i^2$.
	
		\item And the last one is 
		\footnotesize
		\begin{equation}\label{condition}
		z(\r,D_n) \:= \sqrt{\left(  \beta_{\r,n} \dfrac{1}{n-1}\sum_{i=1}^n \left(Y_i - \bar{Y}\right)^2 - \dfrac{1}{n(\r, D_n)-1}\sum_{i = 1}^n \ind{X_i \in \r}\left(Y_i - \bar{Y}_\r\right)^2 \right) },
		\end{equation}
		\normalsize
		where
		\begin{equation}\label{beta_value}
			\beta_{\r, n} = \dfrac{n}{\sum_{\r' \in \S_n}n(\r', D_n)} \max \limits_{A} \# \{ \r'\in \S_n: \r' \cap A \neq \emptyset \},
		\end{equation}
		with the $\max$ is taken upon the set of cells of $\K(\S_n)$ contained in $\r$. It means the set defined by
		$$\left\{A \in \K(\S_n): A \subseteq \r \right\}.$$
	\end{enumerate}

\end{document}